%% file: flapogi.tex
\documentclass{article}

\usepackage{arxiv}

\usepackage[utf8]{inputenc} 
\usepackage[T1]{fontenc}    
\usepackage{hyperref}       
\usepackage{url}            
\usepackage{booktabs}       
\usepackage{amsfonts}       
\usepackage{nicefrac}       
\usepackage{microtype}      
\usepackage[pdftex]{graphicx}
\usepackage{amssymb,amsfonts,amsmath,amscd,amsthm}
\usepackage[printonlyused,withpage]{acronym}
\usepackage{natbib}
\usepackage{mathrsfs}

\graphicspath{{figs/}}
\input{notation_header}

\acrodef{BA}{Bundle Adjustment}
\acrodef{DNN}{Deep Neural Network}
\acrodef{EM}{Expectation Maximization}
\acrodef{GEM}{Generalized Expectation Maximization}
\acrodef{HMM}{Hidden Markov Model}
\acrodef{LDS}{Linear Dynamical System}
\acrodef{LQG}{Linear Quadratic Gaussian}
\acrodef{LQR}{Linear Quadratic Regulator}
\acrodef{LTI}{Linear Time-Invariant}
\acrodef{RTS}{Rauch-Tung-Striebel}
\acrodef{SGD}{Stochastic Gradient Descent}
\acrodef{SLAM}{Simultaneous Localization and Mapping}
\acrodef{RKHS}{Reproducing Kernel Hilbert Space}
\acrodef{SMW}{Sherman-Morrison-Woodbury}
\acrodef{LDU}{Lower Diagonal Upper}

\acrodef{GVI}{Gaussian Variational Inference}
\acrodef{ESGVI}{Exactly Sparse Gaussian Variational Inference}
\acrodef{MAP}{Maximum A Posteriori}
\acrodef{ML}{Maximum Likelihood}
\acrodef{KL}{Kullback-Leibler}
\acrodef{PDF}{Probability Density Function}
\acrodef{NEES}{Normalized Estimation Squared Error}
\acrodef{KF}{Kalman Filter}
\acrodef{VKF}{Variational Kalman Filter}
\acrodef{ISPKF}{Iterated Sigmapoint Kalman Filter}
\acrodef{ESGVI-GN}{ESGVI Gauss-Newton}
\acrodef{ELBO}{Evidence Lower Bound}
\acrodef{NGD}{Natural Gradient Descent}
\acrodef{FIM}{Fisher Information Matrix}
\acrodef{RANSAC}{Random Sample And Consensus}
\acrodef{IRLS}{Iteratively Reweighted Least-Squares}
\acrodef{BRD}{Black-Rangarajan Duality}
\acrodef{VKF}{Variational Kalman Filter}
\acrodef{FLAPOGI}{Fundamental Linear Algebra Problem of Gaussian Inference}
\acrodef{GaBP}{Gaussian Belief Propagation}

\newtheorem{problem}{Problem}
\newtheorem{lemma}{Lemma}
\newtheorem{theorem}{Theorem}

\hypersetup{%
    pdftitle={},
    pdfauthor={},
    pdfkeywords={},
    pdfsubject={},
    pdfstartview=FitH,%
    bookmarks=true,%
    bookmarksopen=true,%
    breaklinks=true,%
    colorlinks=true,%
    linkcolor=black,
    anchorcolor=black,%
    citecolor=black,
    filecolor=black,%
    menucolor=black,
    pagecolor=black,%
    urlcolor=black
}%

\title{Fundamental Linear Algebra Problem \\ of Gaussian Inference}

\author{
 \normalfont Timothy D. Barfoot \\
 Institute for Aerospace Studies \\
 University of Toronto \\
 \texttt{tim.barfoot@utoronto.ca}
}

\begin{document}

\maketitle
\title{Fundamental Linear Algebra Problem of Gaussian Inference}

\begin{abstract}
Underlying many Bayesian inference techniques that seek to approximate the posterior as a Gaussian distribution is a fundamental linear algebra problem that must be solved for both the mean and key entries of the covariance.  Even when the true posterior is not Gaussian (e.g., in the case of nonlinear measurement functions) we can use variational schemes that repeatedly solve this linear algebra problem at each iteration.  In most cases, the question is not whether a solution to this problem exists, but rather how we can exploit problem-specific structure to find it efficiently.  Our contribution is to clearly state the \ac{FLAPOGI} and to provide a novel presentation (using Kronecker algebra) of the not-so-well-known result of \citet{takahashi73} that makes it possible to solve for key entries of the covariance matrix.  We first provide a global solution and then a local version that can be implemented using local message passing amongst a collection of agents calculating in parallel.  Contrary to belief propagation, our local scheme is guaranteed to converge in both the mean and desired covariance quantities to the global solution even when the underlying factor graph is loopy; in the case of synchronous updates, we provide a  bound on the number of iterations required for convergence.  Compared to belief propagation, this guaranteed convergence comes at the cost of additional storage, calculations, and communication links in the case of loops; however, we show how these can be automatically constructed on the fly using only local information.
\end{abstract}

\section{Introduction}

Gaussian probabilistic inference is an important tool in a number of fields including machine learning, computer vision, and robotics.  The aim is to produce a Gaussian belief, $p(\mbf{x}|\mbf{z})$, of the state of the world, $\mbf{x}$, given a prior model, $p(\mbf{x})$, and some evidence, $\mbf{z}$, and hence we often take a Bayesian perspective. Even if the full Bayesian posterior is not Gaussian (e.g., in the case of nonlinear models), we might choose to find the best Gaussian approximation thereof.  

As discussed by \citet{opper09} and more recently by \citet{barfoot_ijrr20} we can start by considering that we want to find a Gaussian approximation, $q(\mbf{x})$, that minimizes the \ac{KL} divergence \citep{kullback51} from the true Bayesian posterior, $p(\mbf{x} | \mbf{z})$:
\begin{equation}
\mbox{KL}(q||p) = - \int^{\mbs{\infty}}_{-\mbs{\infty}} q(\mbf{x}) \ln \left( \frac{p(\mbf{x} | \mbf{z})}{q(\mbf{x})} \right) \, d\mbf{x} .
\end{equation}
Noting that we can factor $p(\mbf{x},\mbf{z}) = p(\mbf{x}|\mbf{z})p(\mbf{z})$ we can equivalently find the $q(\mbf{x})$ that minimizes
\begin{equation}
\label{eq:functional}
V(q) = \mathbb{E}_q[ \phi(\mbf{x})] + \frac{1}{2} \ln \left( |\mbs{\Sigma}^{-1}| \right),
\end{equation}
with $\phi(\mbf{x}) = - \ln p(\mbf{x},\mbf{z})$ and the second term being the well-known expression for the {\em entropy} of a Gaussian.  This functional, $V(q)$, is sometimes referred to as the (negative) {\em Helmholtz free energy} in statistical physics \citep[\S 11.1, p.385] {koller09} or the (negative) \ac{ELBO} in machine learning.  In general, this approach is referred to as {\em variational inference} or {\em variational Bayes} \citep{jordan99,bishop06}.  As we will restrict ourselves to Gaussian approximations of the posterior, we will refer to the approach as \ac{GVI}.

Our Gaussian approximation will take the standard multivariate form,
\begin{equation}
q(\mbf{x}) = \mathcal{N}(\mbs{\mu}, \mbs{\Sigma}) = \frac{1}{\sqrt{(2\pi)^N |\mbs{\Sigma}|}} \exp\left( -\frac{1}{2} (\mbf{x} - \mbs{\mu})^T \mbs{\Sigma}^{-1} (\mbf{x} - \mbs{\mu}) \right), 
\end{equation}
where $|\cdot|$ is the determinant, $\mbs{\mu}$ is the mean, and $\mbs{\Sigma}$ is the covariance.  For many practical robotics and computer vision applications, the dimension of the state, $N$, can become very large and so we are very much concerned with finding efficient solutions that can exploit any structure in our problem.  

Problem-specific structure derives from sparsity in the underlying graphical model of the system.  If we consider that the joint likelihood of the state and data factors, we can write its negative log-likelihood as
\begin{equation}
\phi(\mbf{x}) =  \sum_{k=1}^K \phi_k( \mbf{x}_k ), 
\end{equation}
where $\phi_k(\mbf{x}_k) = - \ln p(\mbf{x}_k,\mbf{z}_k)$ is the $k$th (negative log) factor expression, $\mbf{x}_k$ is a {\em subset} of variables in $\mbf{x}$ associated with the $k$th factor, and $\mbf{z}_k$ is a subset of the data in $\mbf{z}$ associated with the $k$th factor.  Substituting this into~\eqref{eq:functional} we have
\begin{equation}
V(q) = \sum_{k=1}^K \mathbb{E}_{q_k}[ \phi_k(\mbf{x}_k)] + \frac{1}{2} \ln \left( |\mbs{\Sigma}^{-1}| \right),
\end{equation}
where critically the expectation for the $k$th factor reduces to being over $q_k(\mbf{x}_k)$, the {\em marginal} associated with the variables involved in that factor \citep{barfoot_ijrr20}. 

\citet{barfoot_ijrr20} go on to show that the following Newton-like iterative scheme can be used to seek the minimum of $V(q)$ in terms of $\mbs{\mu}$ and $\mbs{\Sigma}^{-1}$:
\begin{subequations}\label{eq:iterstein}
\begin{eqnarray}
\left(\mbs{\Sigma}^{-1}\right)^{(i+1)} & = &  \underbrace{\sum_{k=1}^K \mathbb{E}_{q_k^{(i)}}\left[ \frac{\partial^2}{\partial \mbf{x}_k^T \partial \mbf{x}_k} \phi_k(\mbf{x}_k)\right]}_{\mbf{A}}, \label{eq:iterstein1} \\
\left(\mbs{\Sigma}^{-1}\right)^{(i+1)} \, \delta\mbs{\mu} & = & \underbrace{- \sum_{k=1}^K \mathbb{E}_{q_k^{(i)}}\left[ \frac{\partial}{\partial \mbf{x}_k^T} \phi_k(\mbf{x}_k)\right]}_{\mbf{b}},  \label{eq:iterstein2}   \\
\mbs{\mu}^{(i+1)} & = & \mbs{\mu}^{(i)} + \delta\mbs{\mu},  \label{eq:iterstein3}
\end{eqnarray}
\end{subequations}
where $(i)$ is the iteration index.  This scheme can be shown to be carrying out \ac{NGD} \citep{amari98,Amari2016}, which approximates the Hessian using the \ac{FIM} \citep{barfoot_arxiv20,barfoot_ijrr20,barfoot_amai20}.

The scheme thus proceeds in two alternating phases:
\begin{description}
\item[build the linear system:] We use the current marginals, $q_k(\mbf{x}_k)$, which are also Gaussian, to assemble the two matrices, $\mbf{A}$ and $\mbf{b}$; here we can use Gaussian quadrature to compute the required expectations or even approximate further by evaluating only at the mean, which results in standard \ac{MAP} estimation
\item[solve the linear system:] We solve $\mbf{A} \, \delta\mbs{\mu} = \mbf{b}$ for $\delta\mbs{\mu}$ and then update $\mbs{\mu}$; we must also extract the new marginals, which requires calculating the entries of $\mbf{A}^{-1}$ (the covariance matrix) corresponding to the nonzero entries of $\mbf{A}$ (the inverse covariance matrix)
\end{description}
It is the second phase, `solve the linear system', with which we are concerned in this paper.  The inverse covariance matrix, $\mbf{A}$, is typically quite sparse for key problems such as trajectory estimation, bundle adjustment \citep{brown58}, simultaneous localization and mapping \citep{durrantwhyte06a}, calibration, and pose-graph optimization.  However, the covariance matrix, $\mbf{A}^{-1}$ is usually dense.  \citet{barfoot_ijrr20} further discuss how in computing the marginals for the next `build the linear system' phase we need only compute the entries of $\mbf{A}^{-1}$ associated with the nonzero entries of $\mbf{A}$, although this can also be seen plainly in~\eqref{eq:iterstein}.

The rest of this paper is organized as follows.  Section~\ref{sec:problemsetup} provides the specific linear algebra problem with which we are concerned in this paper.  Section~\ref{sec:globalsolution} uses Kronecker algebra along with so-called elimination and duplication matrices to provide a novel analysis and presentation of a global solution to the problem based on the method of \citet{takahashi73}.  Section~\ref{sec:localsolution} provides a novel method to carry out the global solution in a purely local manner using message passing amongst a collection of agents computing in parallel; this method is guaranteed to converge to the unique global solution in both the mean and desired covariance quantities.  Section~\ref{sec:relatedwork} discusses related work in the literature; we chose to provide related work after the main technical results of the paper so as to allow several detailed connections to be made.  Section~\ref{sec:conclusion} wraps things up and discusses possibilities for future work.

\section{Problem Setup}
\label{sec:problemsetup}

Motivated by the discussion in the introduction, we can begin by stating the problem with which we are concerned in this paper:

\begin{problem} \label{prob:flapogi}
{\em (\acf{FLAPOGI})} Consider the linear system of equations, $\mbf{A} \mbf{x} = \mbf{b}$, where $\mbf{A}$ is $N \times N$, real, symmetric, positive definite, and having a known sparsity pattern; solve for (i) the unique $\mbf{x}$ and (ii) the entries of $\mbf{A}^{-1}$ corresponding to the nonzero entries of $\mbf{A}$.
\end{problem}

Disregarding any sparsity enjoyed by $\mbf{A}$, there is an obvious solution, which is to compute $\mbf{A}^{-1}$ and then (i) take $\mbf{x} = \mbf{A}^{-1} \mbf{b}$ and (ii) extract the desired entries of $\mbf{A}^{-1}$.  This will always be possible given that $\mbf{A} > 0$.  However, when the size of the linear system $N$ is very large the complexity of this brute-force approach is $O(N^3)$.

We can also convert the second part of Problem~\ref{prob:flapogi} into a linear system of equations using the vectorization, $\vec(\cdot)$, and Kronecker product, $\otimes$, operations detailed in Appendix~\ref{app:kronvec}.  We first note that
\begin{equation}
\mbf{A} \mbf{A}^{-1} = \mbf{1},
\end{equation}
where $\mbf{1}$ is the identity matrix.  Then we have
\begin{equation}\label{eq:vecAAinv}
( \mbf{1} \otimes \mbf{A} ) \vec(\mbf{A}^{-1}) = \vec(\mbf{1}),
\end{equation}
where we have used some basic identities found in Appendix~\ref{app:kronvec}.  Since $\mbf{A}^{-1}$ is symmetric (because $\mbf{A}$ is symmetric by assumption), we can remove the redundant entries by defining an {\em elimination} matrix, $\mbf{E}$, to pick the upper-half entries but then reconstitute them using a {\em duplication} matrix, $\mbf{D}$:
\begin{equation}
\mbf{D} \underbrace{\mbf{E} \vec(\mbf{A}^{-1})}_{\mbf{y}} = \vec(\mbf{A}^{-1}),
\end{equation}
which holds for any symmetric matrix.  See Appendix~\ref{app:elimdup} for more on elimination and duplication matrices and their properties.
Inserting this to~\eqref{eq:vecAAinv} and then premultiplying both sides of our equation by $\mbf{E}$ we have
\begin{equation}
\mbf{E} ( \mbf{1} \otimes \mbf{A} )\mbf{D} \, \mbf{y} = \mbf{E} \vec(\mbf{1}).
\end{equation}
Thus, we can solve this system for the unique entries of $\mbf{A}^{-1}$ stored in $\mbf{y}$.  Unfortunately, the coefficient matrix $\mbf{E} ( \mbf{1} \otimes \mbf{A} )\mbf{D}$ does not have any obvious exploitable structure when expressed in this form.  The next section will discuss how to solve
\beqn{flapogi}
\mbf{A} \mbf{x} & = & \mbf{b}, \label{eq:flapogi1} \\
\mbf{E} ( \mbf{1} \otimes \mbf{A} )\mbf{D} \, \mbf{y} & = & \mbf{E} \vec(\mbf{1}), \label{eq:flapogi2}
\eeqn
efficiently using a triangular factorization of $\mbf{A}$.  We will refer to~\eqref{eq:flapogi1} as the {\em primary problem} and~\eqref{eq:flapogi2} as the {\em secondary problem}.  We will see that the order in which place our variables in $\mbf{x}$ and $\mbf{y}$ is important; to avoid confusion, we will refer to the order of $\mbf{x}$ as the {\em primary variable order} and the order of $\mbf{y}$ as the {\em secondary variable order}.

\section{Global Solution}
\label{sec:globalsolution}

We use this section to provide a novel presentation of the result of \citet{takahashi73} (see also \citet{erisman75}) that shows that we do not need to calculate all the entries of $\mbf{A}^{-1}$ to obtain those corresponding to the nonzero entries of $\mbf{A}$.  We will make use of Kronecker algebra and {\em elimination} and {\em duplication} matrices to handle symmetric and triangular matrices \citep{magnus19}.

\subsection{Primary Problem}

We begin by performing a triangular decomposition on $\mbf{A}$,
\begin{equation}
\mbf{A} = \mbf{L}\mbf{S}\mbf{L}^T,
\end{equation}
where $\mbf{L}$ is lower-triangular with ones on its diagonal (and hopefully sparse) and $\mbf{S}$ is diagonal.  This will always be possible given our assumptions that $\mbf{A}$ is symmetric and positive definite.  If $\mbf{A}$ is sparse, so too will be $\mbf{L}$ but possibly with some fill-in.  The amount of fill-in depends on the nature of the underlying graphical model (tree vs. loopy; more on this below) as well as the primary variable order selected for $\mbf{x}$; minimizing the fill-in by reordering $\mbf{x}$ is known to be an NP-hard problem \citep{yannakakis81} and we do not address this here.

To solve the first equation of~\eqref{eq:flapogi}, we can then solve the following two linear systems,
\beqn{}
\mbf{L} \mbf{S} \, \mbf{w} & = & \mbf{b}, \\
\mbf{L}^T \mbf{x} & = & \mbf{w},
\eeqn
where the first can be solved for $\mbf{w}$ through (sparse) forward substitution and the second solved for $\mbf{x}$ through (sparse) backward substitution.  The next section discusses how the triangular decomposition can be reused to solve the second equation of~\eqref{eq:flapogi} for the desired entries of the inverse of $\mbf{A}$.  

\begin{figure}[t]
\includegraphics[width=0.9\textwidth]{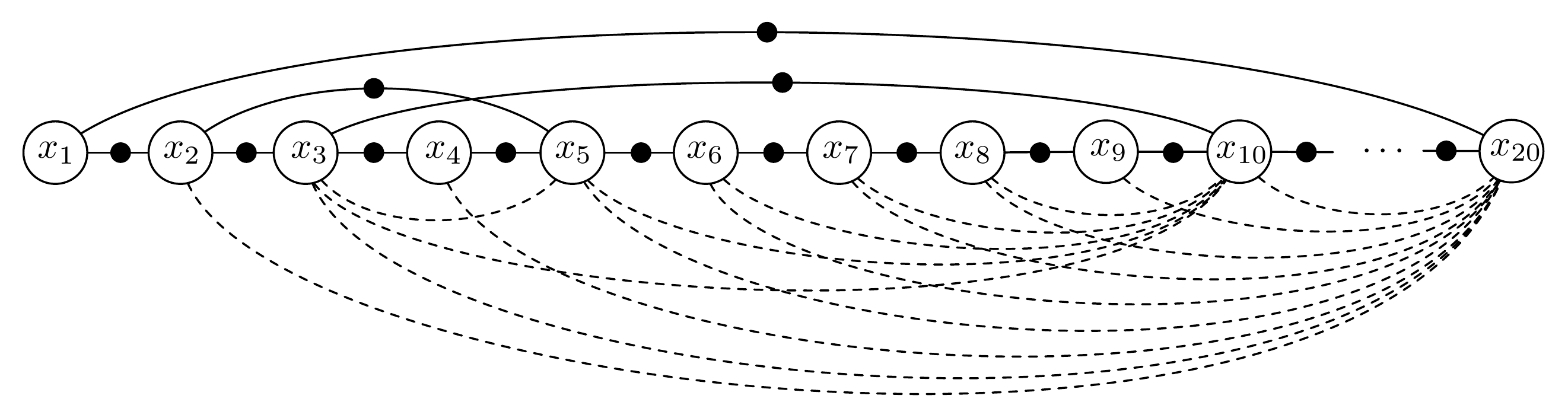}
\centering
\caption{Example loopy factor graph showing $20$ variables (hollow circles) with factors (black filled circles) connecting them (solid lines).   Due to the chosen variable order, when we perform a triangular decomposition there are extra connections that arise (dashed lines) due to the fill-in arising in the triangular factors.  Different variable orders result in different fill-in patterns.}
\label{fig:example1}
\end{figure}

\subsection{Secondary Problem}

To solve the second equation of~\eqref{eq:flapogi}, we can insert the triangular factorization for $\mbf{A}$:
\begin{equation}
\mbf{E} ( \mbf{1} \otimes \mbf{L}\mbf{S}\mbf{L}^T )\mbf{D} \, \mbf{y} = \mbf{E} \vec(\mbf{1}),
\end{equation}
which then factors as
\begin{equation}
\mbf{E} ( \mbf{1} \otimes \mbf{L}\mbf{S} )( \mbf{1} \otimes \mbf{L}^T )\mbf{D} \, \mbf{y} = \mbf{E} \vec(\mbf{1}),
\end{equation}
where we note that $\mbf{L}\mbf{S}$ is lower triangular.  Using one of the identities in Appendix~\ref{app:elimdup} we can rewrite this as
\begin{equation}
\mbf{E} ( \mbf{1} \otimes \mbf{L}\mbf{S} ) \mbf{D} \mbf{E} ( \mbf{1} \otimes \mbf{L}^T )\mbf{D} \, \mbf{y} = \mbf{E} \vec(\mbf{1}),
\end{equation}
and then
\begin{multline}
\mbf{E} ( \mbf{1} \otimes \mbf{L}^T )\mbf{D} \, \mbf{y} = \left( \mbf{E} ( \mbf{1} \otimes \mbf{L}\mbf{S} ) \mbf{D} \right)^{-1} \mbf{E} \vec(\mbf{1})  = \mbf{E} ( \mbf{1} \otimes \mbf{S}^{-1}\mbf{L}^{-1} ) \mbf{D} \mbf{E} \vec(\mbf{1}) \\ = \mbf{E} ( \mbf{1} \otimes \mbf{S}^{-1}\mbf{L}^{-1} ) \vec(\mbf{1}) = \mbf{E} \vec( \mbf{S}^{-1} \mbf{L}^{-1} ) = \mbf{E} \vec( \mbf{S}^{-1} ),
\end{multline}
where we again have used several identities in Appendix~\ref{app:elimdup}.
The key insight of \citet{takahashi73}, that is further explained by \citet{erisman75}, is the last step where we safely remove the $\mbf{L}^{-1}$ from the right-hand side since it is lower-triangular and our elimination matrix is keeping only the upper-half entries of a vectorized matrix.  We can now safely solve
\begin{equation}
\underbrace{\mbf{E} \left( \mbf{1} \otimes \mbf{L}^T \right) \mbf{D}}_{\mbf{C}} \; \mbf{y} = \underbrace{\mbf{E} \vec(\mbf{S}^{-1})}_{\mbf{d}},
\end{equation}
for $\mbf{y}$, which has dimension $M = \frac{1}{2} N(N+1)$.  For some variable order choices, the $M \times M$ coefficient matrix, $\mbf{C}$, is upper triangular (more on this below), so we could simply backward substitute to solve for all the entries of $\mbf{y}$ and hence $\mbf{A}^{-1}$.  However, we would like to avoid this as in many cases we do not need all of the entries.

We can work out a useful expression for the entries of our coefficient matrix, $\mbf{C}$.  The $ij,k\ell$ entry is given by
\begin{equation}
C_{ij,k\ell} = \vec(\mbf{1}_{ji})^T \left( \mbf{1} \otimes \mbf{L}^T \right) \vec( (1 - \delta_{k\ell}) \mbf{1}_{k\ell} + \mbf{1}_{\ell k} ),
\end{equation}
where we use the detailed definitions of $\mbf{E}$ and $\mbf{D}$ from the Appendix; $\mbf{1}_{ab}$ is an $N \times N$ matrix of zeros except for the $ab$ entry, which is one.  
We can write
\begin{equation}\label{eq:L}
\mbf{L} = \sum_{m,n} L_{mn} \mbf{1}_{mn}.
\end{equation}
Inserting this we can write
\begin{equation}
C_{ij,k\ell} = \sum_{m,n} L_{mn} \vec(\mbf{1}_{ji})^T \left( \mbf{1} \otimes \mbf{1}_{mn}^T \right) \vec( (1 - \delta_{k\ell}) \mbf{1}_{k\ell} + \mbf{1}_{\ell k} ),
\end{equation}
and then we can apply basic properties of $\vec(\cdot)$ and $\otimes$ to rewrite this as
\begin{equation}
C_{ij,k\ell} = \sum_{m,n} L_{mn} \mbox{tr} \left( \mbf{1}_{ij} \mbf{1}_{nm} \left( (1 - \delta_{k\ell}) \mbf{1}_{k\ell} + \mbf{1}_{\ell k} \right)  \right).
\end{equation}
We can now use the fact that 
\begin{equation} \label{eq:1ij1nm}
\mbf{1}_{ab} \mbf{1}_{cd} = \left\{ \begin{tabular}{cl} $\mbf{1}_{ad}$ & \mbox{if $b=c$} \\ $\mbf{0}$ & \mbox{otherwise} \end{tabular} \right. ,
\end{equation}
to rewrite our expression as
\begin{equation}
C_{ij,k\ell} = \sum_{m,j} L_{mj} \mbox{tr} \left( \mbf{1}_{im} \left( (1 - \delta_{k\ell}) \mbf{1}_{k\ell} + \mbf{1}_{\ell k} \right)  \right).
\end{equation}
Applying~\eqref{eq:1ij1nm} again we have
\begin{equation}
C_{ij,k\ell} =   L_{kj} \mbox{tr} (1 - \delta_{k\ell}) \left( \mbf{1}_{i\ell} \right) + L_{\ell j} \mbox{tr}\left( \mbf{1}_{ik} \right),
\end{equation}
or
\begin{equation}
C_{ij,k\ell} =   L_{kj} (1 - \delta_{k\ell})\delta_{i\ell} + L_{\ell j} \delta_{ik},
\end{equation}
where $\mbox{tr} \left( \mbf{1}_{ab} \right)  = \delta_{ab}$.  Working through the combinations of the various Kronecker delta functions we finally have
\begin{equation}\label{eq:Cijkl}
C_{ij,k\ell} = \left\{\begin{tabular}{ll}  
$L_{\ell j}$ & \mbox{if $i = k$} \\
$L_{k j}$ & \mbox{if $i = \ell$} \\
$0$ & \mbox{otherwise}
\end{tabular}\right. .
\end{equation}
We will use this expression to show various properties of $\mbf{C}$ in the next section.  

We can also work out an expression for the entries of $\mbf{d}$.  The $ij$ entry is given by
\begin{equation}
d_{ij} = \vec(\mbf{1}_{ji})^T \vec(\mbf{S}^{-1}) = \mbox{tr}( \mbf{1}_{ij} \mbf{S}^{-1} ).
\end{equation}
Inserting $\mbf{S}^{-1} = \sum_{kk} S_{kk}^{-1} \mbf{1}_{kk}$ we have
\begin{equation}\label{eq:dij}
d_{ij} = \vec(\mbf{1}_{ji})^T \vec(\mbf{S}^{-1}) = \sum_{kk} S_{kk}^{-1} \mbox{tr}( \mbf{1}_{ij} \mbf{1}_{kk} ) = S_{jj}^{-1} \mbox{tr}( \mbf{1}_{ij}) = S_{jj}^{-1} \delta_{ij},
\end{equation}
which makes sense since $\mbf{S}$ is diagonal.

\subsection{How to Solve Only for Desired Entries of $\mbf{A}^{-1}$}

If $\mbf{A}$ is quite sparse then so too will be $\mbf{L}$ and hence many of its entries, $L_{ij}$, will be zero; we assume we know the sparsity pattern of $\mbf{L}$ as a by-product of the triangular decomposition.  We will refer to the set of indices of $\mbf{L}$ corresponding to nonzero entries as $\mathcal{L}$, while $\bar{\mathcal{L}}$ will be the complementary set of indices in the lower triangle (corresponding to zero entries of the lower triangle of $\mbf{L}$).

Owing to the way the triangular decomposition works, we have the following lemma:
\begin{lemma}\label{lem:box}
{\em (Four Corners of a Box)} If ${ij}, {kj} \in \mathcal{L}$ with $i < k$ then ${ki} \in \mathcal{L}$.  
\end{lemma}
\begin{proof}
Discussed by \citet{erisman75}.  Note, the fourth corner of the box is ${ii} \in \mathcal{L}$.
\end{proof}

We are now ready to prove a key result regarding $\mbf{C}$.  Consider the situation where $ij \in \mathcal{L}$ and $k \ell \in \bar{\mathcal{L}}$; if we can show that the entry at this location is zero, then it means that the entries of $\mbf{y}$ corresponding to the nonzero entries of $\mbf{L}$ can be calculated without ever calculating the remaining entries.  

\begin{theorem} \label{thm:closure}
{\em \citep{takahashi73} (Closure of $\mathcal{L}$)} If $ij \in \mathcal{L}$ and $k \ell \in \bar{\mathcal{L}}$ then $C_{ij, k\ell} = 0$.
\end{theorem}
\begin{proof}
From the definitions of the indices, we know $j \leq i$ and $\ell < k$; cannot have $\ell = k$ since then $k \ell = kk \in \mathcal{L}$.   Assume $C_{ij,kl} \neq 0$.  According to~\eqref{eq:Cijkl}, this requires one of two conditions:

Case (i):  $i = \ell$ whereupon $C_{ij,k\ell} = L_{kj}$.  This will still be zero unless $kj \in \mathcal{L}$.    Lemma~\ref{lem:box} tells us that if $ij, kj \in \mathcal{L}$ then so must be $ki$ if $i < k$, which is true since $i = \ell < k$.  Furthermore, since $i = \ell$ it must be that $ki = k \ell \in \mathcal{L}$, which is a contradiction.  

Case (ii):  $i = k$ whereupon $C_{ij,k\ell} = L_{\ell j}$. This will still be zero unless $\ell j \in \mathcal{L}$.  Lemma~\ref{lem:box} tells us that if $ij, \ell j \in \mathcal{L}$ then so must be $i \ell$ if $\ell < i$, which is true since $\ell < k = i$.  Furthermore, since $i = k$ it must be that $i \ell = k \ell \in \mathcal{L}$, which is a contradiction.  

Both cases result in contradictions, therefore $C_{ij,k\ell} = 0$.
\end{proof}

\begin{figure}[t]
\includegraphics[width=\textwidth]{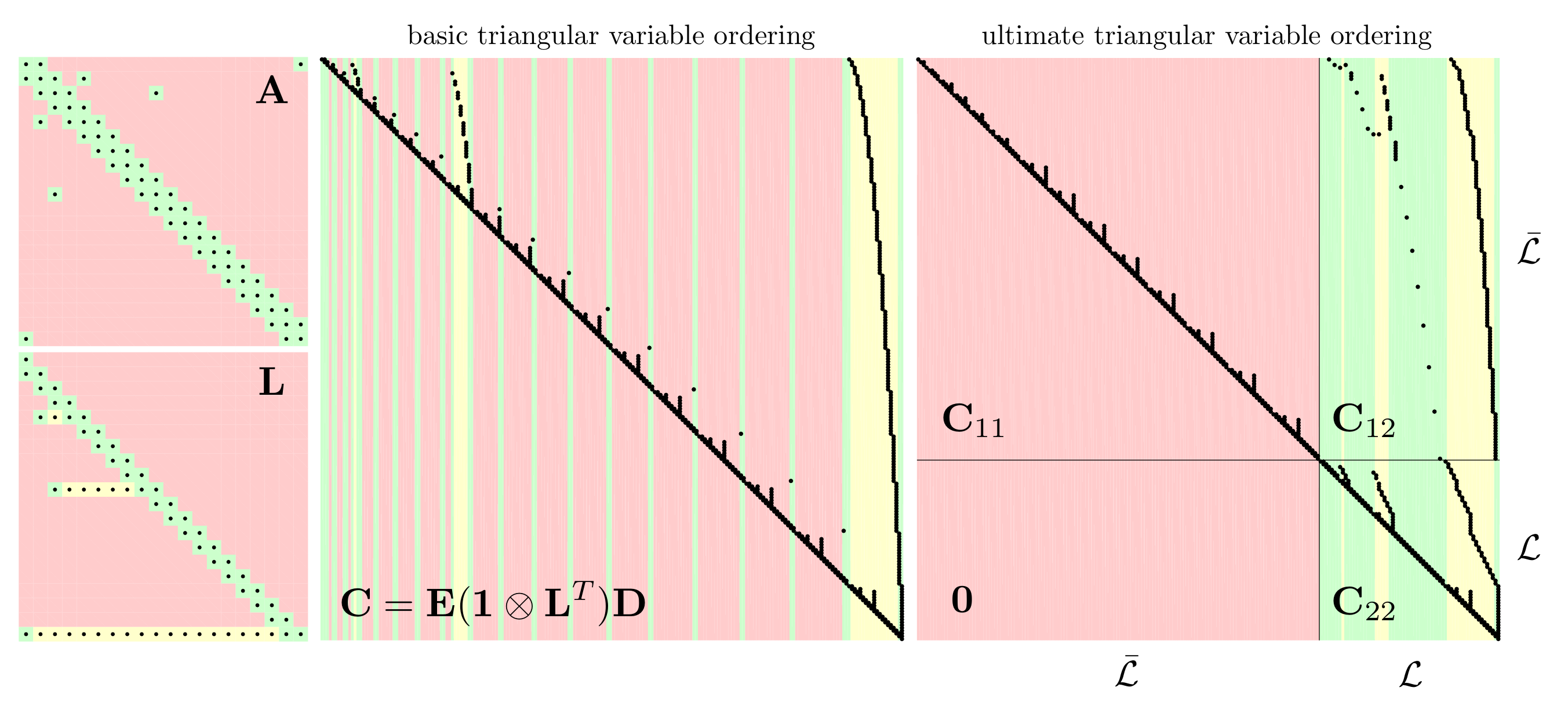}
\caption{Example sparsity patterns for the $\mbf{A}$ arising from the factor graph in Figure~\ref{fig:example1}, its triangular factor $\mbf{L}$, and the coefficient matrix $\mbf{C}$ of the sparse linear system for $\mbf{A}^{-1}$.  Two different secondary variable orders are shown for $\mbf{C}$.  Green corresponds to the desired entries of $\mbf{A}^{-1}$ that are nonzero in $\mbf{A}$, yellow to the extra nonzero fill-in entries in $\mbf{L}$, and red for the remaining entries.  Under the ultimate triangular variable order, the bottom-left block of $\mbf{C}$ is zero when the variables are partitioned into $\bar{\mathcal{L}}$ and $\mathcal{L}$.}
\label{fig:exampleC}
\end{figure}

\subsection{Triangular Secondary Variable Orders}

Depending on the secondary variable order chosen to construct the elimination/duplication matrices, $\mbf{E}$ and $\mbf{D}$, the structure of $\mbf{C} = \mbf{E} \left( \mbf{1} \otimes \mbf{L}^T \right) \mbf{D}$ will vary.  Conveniently, there are variable orders that can be chosen that make $\mbf{C}$ upper triangular.  These naturally have to do with the structure of $\mbf{L}$.  

First, we define the following {\em basic triangular} order
\begin{equation}
\mbf{v} = ( 11, 21, 22, 31, 32, 33, \ldots, NN ),
\end{equation}
where the first number is the $i$ and the second the $j$. This is simply the list of indices of a lower-triangular matrix from the top row to the bottom and left to right within each row.  We then have the following:

\begin{theorem} \label{thm:triangular}
{\em (Basic Triangular Variable Order)} Under the variable order, $\mbf{v}$, the matrix $\mbf{C}$ is upper triangular.
\end{theorem}
\begin{proof}
For $\mbf{C}$ to be upper triangular, we must show that $C_{m,n} = 0$ for $n < m$.  If we let $m = ij$ and $n = k \ell$ then for $n < m$ we must either have (1) $k < i$ or (2) $k=i$ and $\ell < j$, according to our variable order.  From the definitions of the indices, we also know $j \leq i$ and $\ell \leq k$.  Assume $C_{ij,k\ell} \neq 0$.  According to~\eqref{eq:Cijkl}, this requires one of two conditions:

Case (i):  $i = \ell$.  For (1) we have $\ell \leq k < i$ so we cannot have $i = \ell$.  For (2) we have $\ell < j \leq i$ so again cannot have $i = \ell$.  Both cases result in contradictions so that $C_{ij,k\ell} = 0$.

Case (ii):  $i = k$, whereupon $C_{ij,k\ell} = L_{\ell j}$. For (1) we have $k < i$ so cannot have $i = k$; the contradiction implies $C_{ij,k\ell} = 0$.  For (2) we have $\ell < j$ so that $\ell j \in \bar{\mathcal{L}}$ and hence  $C_{ij,k\ell} = L_{\ell j} = 0$.

All cases result in $C_{ij,k\ell} = 0$.
%
%
%
\end{proof}

The basic variable order, $\mbf{v}$, is not the only one that results in an upper-triangular $\mbf{C}$.  We also have the following:

\begin{lemma}\label{lem:triperm}
{\em (Triangular Permutations)} Starting from any variable order, $\mbf{v}$, that results in $\mbf{C}$ upper triangular, we can swap the variables in adjacent positions $m$ and $m+1$ provided $m = ij \in \mathcal{L}$ and $m + 1 = k \ell \in \bar{\mathcal{L}}$ to create a new variable order, $\mbf{v}^\prime$, that also results in $\mbf{C}$ upper triangular.
\end{lemma}
\begin{proof}
This is a direct consequence of Theorem~\ref{thm:closure}.
\end{proof}

\begin{figure}[t]
\includegraphics[width=\textwidth]{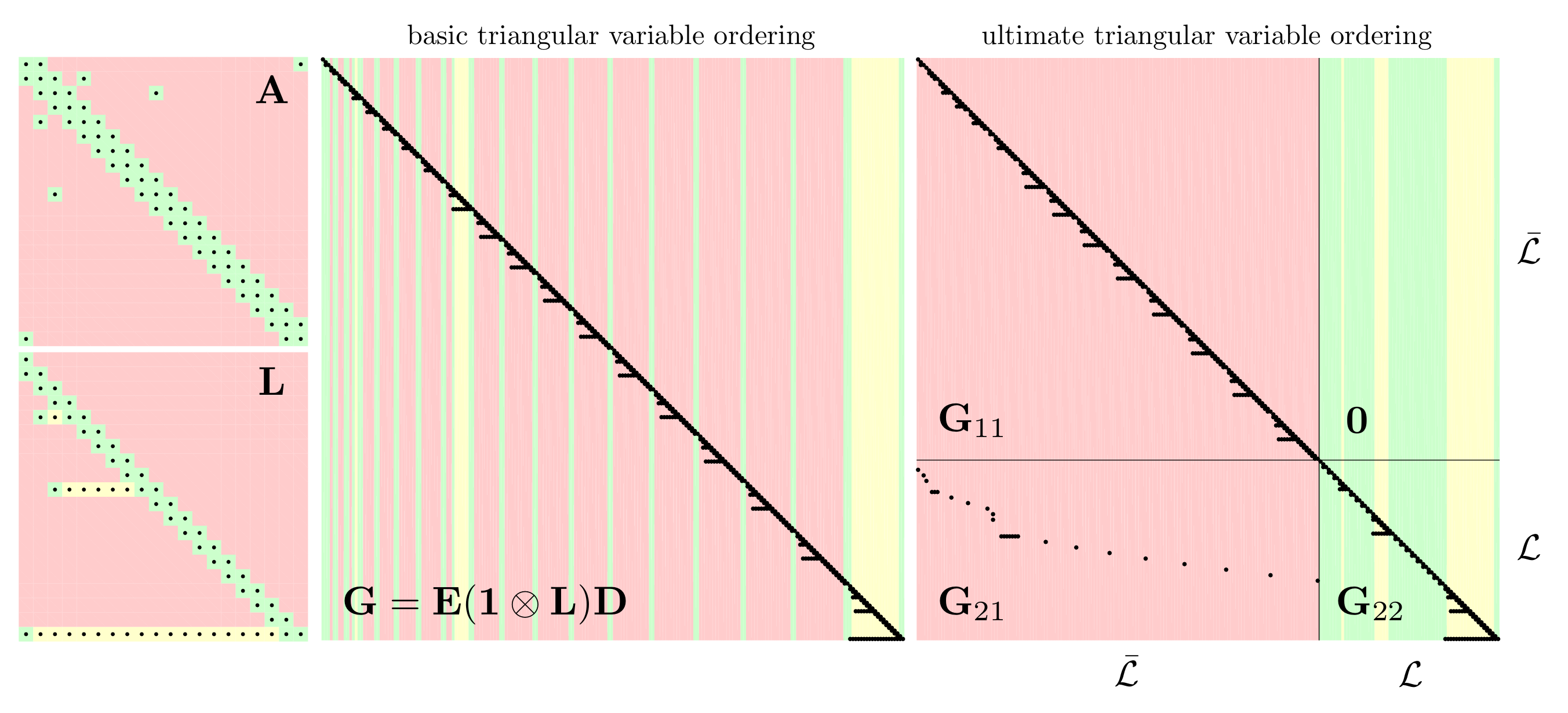}
\caption{Example sparsity patterns for an $\mbf{A}$, its triangular factor $\mbf{L}$, and the matrix $\mbf{G}$.  Two different secondary variable orders are shown for $\mbf{C}$.  Green corresponds to the desired entries of $\mbf{A}^{-1}$ that are nonzero in $\mbf{A}$, yellow to the extra nonzero entries in $\mbf{L}$, and red for the remaining entries.  Under the ultimate triangular variable order the upper-right block of $\mbf{G}$ is zero when the variables are partitioned into $\mathcal{L}$ and $\bar{\mathcal{L}}$.}
\label{fig:exampleG}
\end{figure}

Finally, we have the following key result:

\begin{theorem}\label{thm:ultimateorder}
{\em (Ultimate Triangular Variable Order)} Let $\mbf{v}^\prime$ be the variable order that starts with the basic order, $\mbf{v}$, and applies Lemma~\ref{lem:triperm} until, {\em ultimately}, no more swaps are possible.  Then using $\mbf{v}^\prime$ to build the elimination and duplication matrices, our linear system for the inverse becomes
\begin{equation}\label{eq:utvoss}
\bbm \mbf{C}_{11} & \mbf{C}_{12} \\ \mbf{0} & \mbf{C}_{22} \ebm \bbm \mbf{y}_1 \\ \mbf{y}_2 \ebm = \bbm \mbf{d}_1 \\ \mbf{d}_2 \ebm,
\end{equation}
with $\mbf{C}_{11}$, $\mbf{C}_{22}$ upper triangular, the `1' partition containing all the variables in $\bar{\mathcal{L}}$, and the `2' partition containing those in $\mathcal{L}$.
\end{theorem}
\begin{proof}
The basic variable order, $\mbf{v}$, results in an upper-triangular $\mbf{C}$ according to Theorem~\ref{thm:triangular}.  Then by performing the swaps according to Lemma~\ref{lem:triperm} we maintain $\mbf{C}$ upper triangular until no more swaps are possible. Theorem~\ref{thm:closure} then tells us the lower-left block matrix is $\mbf{0}$.
\end{proof}

Figure~\ref{fig:exampleC} provides an example of the sparsity pattern of $\mbf{C}$.  While we would never actually construct the full linear system, $\mbf{C} \mbf{y} = \mbf{d}$, it provides some insight into efficient methods of solving for the desired entries of $\mbf{A}^{-1}$.  
The ultimate triangular variable order is particularly appealing as we can simply perform a backward substitution and start by solving for all the entries in $\mathcal{L}$, or in other words the entries of $\mbf{A}^{-1}$ corresponding to the nonzero entries of $\mbf{L}$ (and hence $\mbf{A}$).  We can also continue to back substitute for additional entries of the inverse as desired.

We will later also make use of the matrix $\mbf{G} = \mbf{E}( \mbf{1} \otimes \mbf{L}) \mbf{D}$, which is lower-triangular for any of our triangular variable orders.  Figure~\ref{fig:exampleG} provides an example sparsity pattern for $\mbf{G}$ under both the basic and ultimate triangular variable orders.  Similarly to~\eqref{eq:Cijkl}, we have the following expression for the entries of $\mbf{G}$:
\begin{equation}\label{eq:Gijkl}
G_{ij,k\ell} = \left\{\begin{tabular}{ll}  
$L_{j\ell}$ & \mbox{if $i = k$} \\
$0$ & \mbox{otherwise}
\end{tabular}\right. .
\end{equation}
We can state the following theorems relating to $\mbf{G}$:

\begin{theorem} \label{thm:triangular2}
{\em (Basic Triangular Variable Order II)} Under the variable order, $\mbf{v}$, the matrix $\mbf{G} = \mbf{E}( \mbf{1} \otimes \mbf{L}) \mbf{D}$ is lower triangular.
\end{theorem}
\begin{proof}
Similar to the proof of Theorem~\ref{thm:triangular}.
\end{proof}

\begin{lemma}\label{lem:triperm2}
{\em (Triangular Permutations II)} Starting from any variable order, $\mbf{v}$, that results in $\mbf{G}$ lower triangular, we can swap the variables in adjacent positions $m$ and $m+1$ provided $m = ij \in \mathcal{L}$ and $m + 1 = k \ell \in \bar{\mathcal{L}}$ to create a new variable order, $\mbf{v}^\prime$, that also results in $\mbf{G}$ lower triangular.
\end{lemma}
\begin{proof}
This is a direct consequence of Theorem~\ref{thm:closure}.
\end{proof}

\begin{theorem}\label{thm:ultimateorder2}
{\em (Ultimate Triangular Variable Order II)} Let $\mbf{v}^\prime$ be the variable order that starts with the basic order, $\mbf{v}$, and applies Lemma~\ref{lem:triperm2} until, {\em ultimately}, no more swaps are possible.  Then using $\mbf{v}^\prime$ to build the elimination and duplication matrices, the matrix $\mbf{G}$ is lower triangular.
\end{theorem}
\begin{proof}
Follows from applying Lemma~\ref{lem:triperm2} until no more swaps are possible.
\end{proof}

\begin{figure}[t]
\includegraphics[width=\textwidth]{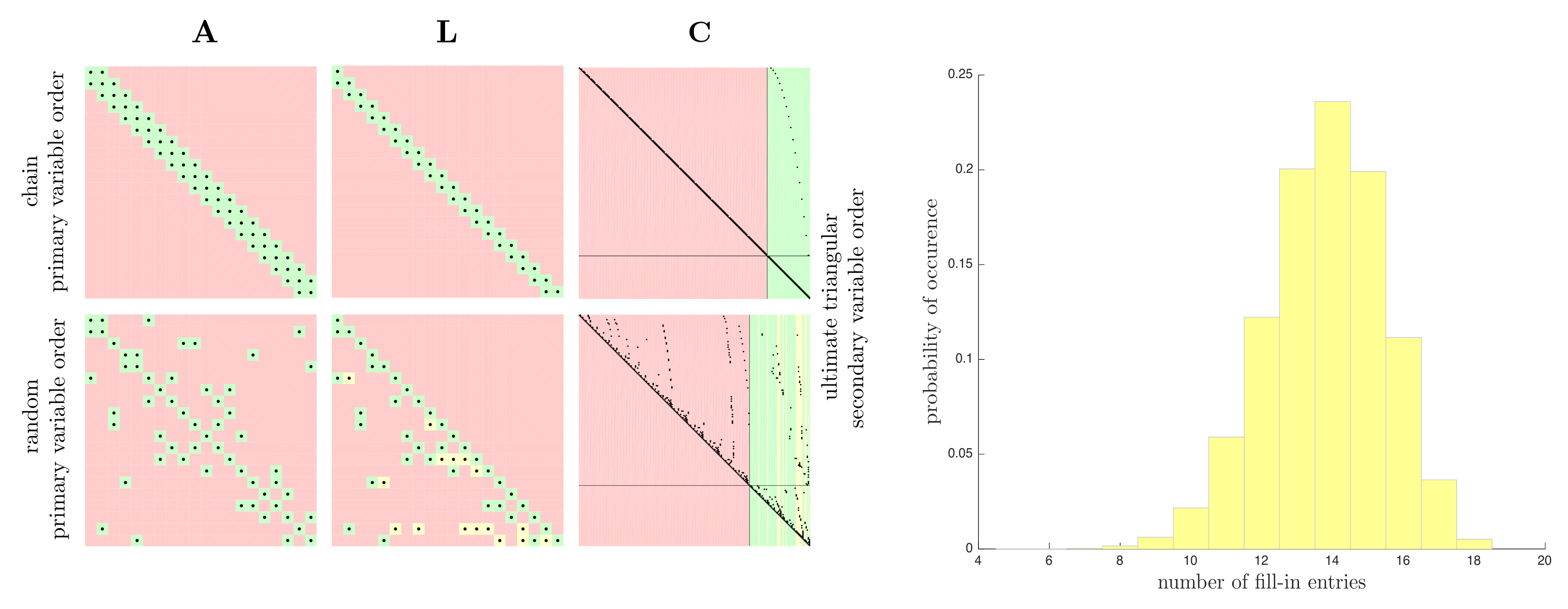}
\caption{Even with a factor graph that is a simple chain, it is possible to have fill-in occur in the triangular factor, $\mbf{L}$.  In this example of 20 variables in a chain, the chain primary variable order has no fill-in while a random variable order does have fill-in (16 yellow entries in this case) and therefore requires solving for more entries of the inverse of $\mbf{A}$ in the secondary problem.  The histogram shows the empirical probability of number of fill-in entries, where zero corresponds to the chain ordering.  We see that a random selection of the primary variable order is highly unlikely to result in zero fill-in but also to result in complete fill-in, which would be 181 entries in this case.}
\label{fig:chain_fillin}
\end{figure}

\subsection{Alternate Expression for Triangular Factorization}
\label{sec:iterL}

Our desired triangular factorization of $\mbf{A}$ is
\begin{equation}
\mbf{L} \mbf{S} \mbf{L}^T = \mbf{A},
\end{equation}
where $\mbf{A}$ is sparse and known and we need to solve for lower-triangular $\mbf{L}$ and diagonal $\mbf{S}$.  Vectorizing both sides we have
\begin{equation}
\vec(\mbf{L} \mbf{S} \mbf{L}^T) = \vec(\mbf{A}).
\end{equation}
Applying one of the Kronecker product identities on the left we have
\begin{equation}
\left( \mbf{1} \otimes \mbf{L} \right) \vec(\mbf{S} \mbf{L}^T) = \vec(\mbf{A}).
\end{equation}
Premultiplying by the elimination matrix $\mbf{E}$ we get
\begin{equation}
\mbf{E} \left( \mbf{1} \otimes \mbf{L} \right) \vec(\mbf{S} \mbf{L}^T) = \mbf{E} \vec(\mbf{A}),
\end{equation}
and then using one of the identities from Appendix~\ref{app:elimdup} we arrive at
\begin{equation}
\underbrace{\mbf{E} \left( \mbf{1} \otimes \mbf{L} \right) \mbf{D}}_{\mbf{G}} \underbrace{\mbf{E} \vec(\mbf{S} \mbf{L}^T)}_{\mbf{v}} = \underbrace{\mbf{E} \vec(\mbf{A})}_{\mbf{h}},
\end{equation}
where as stated in Theorem~\ref{thm:ultimateorder2} and shown in Figure~\ref{fig:exampleG} the matrix $\mbf{G}$ is lower triangular.  Using the ultimate triangular variable order, we can partition our variables into the $\bar{\mathcal{L}}$ and $\mathcal{L}$ sets such that
\begin{equation}
\bbm \mbf{G}_{11} & \mbf{0} \\ \mbf{G}_{21} & \mbf{G}_{22} \ebm \bbm \mbf{0} \\ \mbf{v}_2 \ebm = \bbm \mbf{0} \\ \mbf{h}_2 \ebm,
\end{equation}
where importantly we have $\mbf{v}_1 = \mbf{0}$ and $\mbf{h}_1 = \mbf{0}$; this is because the chosen secondary variable order deliberately puts all of the nonzero elements of $\mbf{S}\mbf{L}^T$ (and the upper-half of $\mbf{A}$) into the second partition leaving the first partition full of zeros.  From this, we see that we only need to solve 
\begin{equation}
\mbf{G}_{22} \mbf{v}_2 = \mbf{h}_2,
\end{equation}
for $\mbf{v}_2$, which will contain all the nonzero entries of $\mbf{S}\mbf{L}^T$; this can be done by backward substitution since $\mbf{G}_{22}$ is lower triangular.  We can therefore reconstitute $\mbf{S}\mbf{L}^T$ from $\mbf{v}_2$ and then $\mbf{S} = \mbox{diag}( \mbf{S} \mbf{L}^T)$ and $\mbf{L} = ( \mbf{S} \mbf{L}^T )^T \mbf{S}^{-1} $.  However, we note that $\mbf{G}_2$ also depends on $\mbf{L}$, which means as we solve from top to bottom we need to be substituting in the latest entry of $\mbf{L}$ to $\mbf{v}_2$.  We start with initial guesses $\mbf{L} = \mbf{1}$ and $\mbf{S} = \mbox{diag}(\mbf{A})$.  We will later exploit this formulation of the triangular factorization to turn it into an iterative scheme wherein we build $\mbf{G}_{22}$ from the previous $\mbf{L}$, solve for $\mbf{v}_2$ and hence the next $\mbf{S}$ and next $\mbf{L}$, and repeat until converged; this is guaranteed to converge since the original forward substitution occurs in the process.

\subsection{Trees vs. Loops}

It is well known that inference for factor graphs with loops can be harder than inference for factor graphs without loops (i.e., trees).  The reason is that for trees, it is always possible to find a primary variable order for which there is no fill-in in the resulting triangular factors.  However, it is not true that all primary variable orders result in no fill-in even when the graph is a tree or even a chain.  Figure~\ref{fig:chain_fillin} shows what happens for two different primary variable orders, chain and random, for a simple chain of variables.  A good heuristic to order the primary variables, even in the case of a loopy graph, is to select a spanning tree for the factor graph and order the variables by a search along this graph.  Our global solution will find the correct solution even if the underlying factor graph has loops, at the expense of some overhead calculations arising from the fill-in of $\mbf{L}$ for the chosen primary variable order.

\subsection{Summary}

Summarizing the global solution to Problem~\ref{prob:flapogi}, we do the following:
\beqn{global}
\underbrace{\mbf{E} \left( \mbf{1} \otimes \mbf{L} \right) \mbf{D}}_{\mbf{G}} \, \mbf{v}  &  = &  \underbrace{\mbf{E} \vec(\mbf{A})}_{\mbf{h}}  \qquad \quad\, \mbox{(solve for $\mbf{v}_2$ by sparse forward substitution)} \\
\mbf{L} \mbf{S} \, \mbf{w} & = & \mbf{b}, \qquad \qquad \qquad\mbox{(solve for $\mbf{w}$ by sparse forward substitution)} \\
\mbf{L}^T \mbf{x} & = & \mbf{w}, \hspace{0.01in}\qquad \qquad\quad\;\;\,\mbox{(solve for $\mbf{x}$ by sparse backward substitution)} \\
\underbrace{\mbf{E} \left( \mbf{1} \otimes \mbf{L}^T \right) \mbf{D}}_{\mbf{C}} \, \mbf{y} & = & \underbrace{\mbf{E} \vec(\mbf{S}^{-1})}_{\mbf{d}}, \qquad \mbox{(solve for $\mbf{y}_2$ by sparse backward substitution)}
\eeqn
where we assume the ultimate triangular (secondary) variable order for $\mbf{v}$ (so that $\mbf{v}_2$ contains the entries of $\mbf{L}$ and $\mbf{S}$) and $\mbf{y}$ (so that $\mbf{y}_2$ contains the desired entries of $\mbf{A}^{-1}$).  We note that all four of these linear systems have a left-hand side coefficient matrix that is triangular; the first two are upper triangular and can be solved using sparse forward substitution while the last two are lower triangular and can be solved by sparse backward substitution.  At implementation, we would need to use a sparse matrix library to minimize storage and compute.

\section{Local Solution}
\label{sec:localsolution}

We use this section to show how we can turn the global solution of the previous section into a local message-passing scheme that solves both the primary and secondary systems with guaranteed convergence.  We begin by showing how to solve each of the systems in~\eqref{eq:global} in an iterative manner (with provable convergence) and then show how to deploy the iterative scheme using local message passing amongst a group of $N$ agents performing their calculations in parallel.

%

\subsection{Triangular Factorization}

A critical step is our ability to carry out the triangular factorization of $\mbf{A}$ in a decentralized manner.  We will build on the results of Section~\ref{sec:iterL} above.  Our goal is to solve
\begin{equation}
\mbf{G}_{22} \mbf{v}_2 = \mbf{h}_2,
\end{equation}
for $\mbf{v}_2$ using only local message passing.  Consider the $ij$ row of this system:
\begin{equation}
\sum_{k\ell} G_{ij,k\ell} v_{k \ell} = h_{ij}.
\end{equation}
We know that $\mbf{G}_{22}$ is lower triangular so we can restrict the summation:
\begin{equation}
\sum_{k \ell \leq ij} G_{ij,k\ell} v_{k \ell} = h_{ij}.
\end{equation}
We can then break the summation into two parts:
\begin{equation}
G_{ij,ij} v_{ij} + \sum_{k\ell < ij} G_{ij,k\ell} v_{k \ell} = h_{ij}.
\end{equation}
Then using~\eqref{eq:Gijkl} we have
\begin{equation}
\underbrace{L_{jj}}_{1} v_{ij} + \sum_{\ell < j} L_{j\ell} v_{i \ell} = \underbrace{h_{ij}}_{A_{ij}}.
\end{equation}
Isolating for $v_{ij}$ we have
\begin{equation}
v_{ij} =  A_{ij} - \sum_{\ell < j} L_{j\ell} v_{i \ell}.
\end{equation}
It is not too difficult to show that
\begin{equation}
v_{ij} = L_{ij} S_{jj},
\end{equation}
using the definition of $\mbf{v}$.  Substituting this we have
\begin{equation}
L_{ij} S_{jj} =  A_{ij} - \sum_{\ell < j} L_{j\ell} L_{i\ell} S_{\ell \ell}.
\end{equation}
We can break this into two parts:
\beqn{}
S_{ii} & = &  A_{ii} - \sum_{\ell < i} L_{i\ell}^2 S_{\ell \ell}, \\
L_{ij} & = & \bigl( A_{ij} - \sum_{\ell < j} L_{j\ell} L_{i\ell} S_{\ell \ell} \bigr) / S_{jj} \qquad \mbox{with $i  > j$},
\eeqn
where we make use of the fact that $L_{ii} = 1$.
We can use these as an iterative update for calculating all the required entries of $\mbf{S}$ and $\mbf{L}$ in a sparse way.  

\subsection{Primary System}

Now that we have the triangular factorization, we can easily solve the primary system using the usual forward and backward substitution, but we can do this in an iterative way as well.  To solve $\mbf{L}\mbf{S} \, \mbf{w} = \mbf{b}$ we can consider row $i$:
\begin{equation}
\sum_j L_{ij} S_{jj} w_j = b_i.
\end{equation}
Due to the structure of $\mbf{L} \mbf{S}$ we can equivalently write this as
\begin{equation}
\underbrace{L_{ii}}_{1}S_{ii} w_i + \sum_{j< i} L_{ij} S_{jj} w_j = b_i,
\end{equation}
and then isolating for $w_i$ we have
\begin{equation}
w_i = \bigl( b_i - \sum_{j< i} L_{ij} S_{jj} w_j \bigr)/S_{ii}.
\end{equation}
We can use this as an iterative scheme, initialized with $w_i = b_i$.

To solve $\mbf{L}^T \mbf{x} = \mbf{w}$ we can again consider row $i$:
\begin{equation}
\sum_j L_{ji} x_j = w_i.
\end{equation}
Due to the structure of $\mbf{L}^T$ we can equivalently write this as
\begin{equation}
\underbrace{L_{ii}}_{1} x_i + \sum_{j > i} L_{ji} x_j = w_i,
\end{equation}
and then isolating for $w_i$ we have
\begin{equation}
x_i = w_i - \sum_{j > i} L_{ji} x_j.
\end{equation}
We can use this as an iterative scheme, initialized with $x_i = b_i$.

\subsection{Secondary System}

The last step is to solve $\mbf{C}_{22} \mbf{y}_2 = \mbf{d}_2$ for the desired entires of the inverse of $\mbf{A}$ stored in $\mbf{y}_2$.  We again consider row $ij$ of this system of equations:
\begin{equation}
\sum_{k\ell} C_{ij,k\ell} y_{k\ell} = d_{ij}.
\end{equation}
We know that $\mbf{C}_{22}$ is upper triangular so we can restrict the summation:
\begin{equation}
\sum_{k\ell \geq ij} C_{ij,k\ell} y_{k\ell} = d_{ij}.
\end{equation}
Since we are using the ultimate triangular variable order, to have $k\ell \geq ij$ we must have (i) $i < k$ or (ii) $i = k$ and $j \leq \ell$.  The nature of the indices is also such that $j \leq i$ and $\ell \leq k$.  Then using~\eqref{eq:Cijkl} and~\eqref{eq:dij} we can write
\begin{equation}
\sum_{i \geq \ell \geq j} L_{\ell j} y_{i\ell} + \sum_{k > i} L_{k j} y_{k i} = S_{jj}^{-1} \delta_{ij}.
\end{equation}
The first term comes from $i = k$ so then $j \leq \ell \leq k = i$.  The second term comes from $i = \ell$ so then $j \leq i = \ell < k$.  Breaking the first summation into two parts we have
\begin{equation}
\underbrace{L_{jj}}_{1} y_{ij} + \sum_{i \geq \ell > j} L_{\ell j} y_{i\ell} + \sum_{k > i} L_{k j} y_{k i} = S_{jj}^{-1} \delta_{ij}, 
\end{equation}
and then isolating for $y_{ij}$ we arrive at
\begin{equation}
y_{ij} = S_{jj}^{-1} \delta_{ij} - \sum_{i \geq \ell > j} L_{\ell j} y_{i\ell} - \sum_{k > i} L_{k j} y_{k i}, \qquad \mbox{with $i  \geq j$}.
\end{equation}
We can use this as an iterative update for calculating all the entries of $\mbf{y}_2$.

\subsection{Summary and Convergence}

At the single-entry level we can carry out the following iterative updates (synchronously or asyncronously) to solve both the primary and secondary systems:
\beqn{iterupdate}
S_{ii} & \leftarrow &  A_{ii} - \sum_{\ell < i} L_{i\ell}^2 S_{\ell \ell}, \\
L_{ij} & \leftarrow & \bigl( A_{ij} - \sum_{\ell < j} L_{j\ell} L_{i\ell} S_{\ell \ell} \bigr) / S_{jj}, \qquad \mbox{with $i  > j$ and $ij \in \mathcal{L}$}, \\
w_i & \leftarrow & \bigl( b_i - \sum_{j< i} L_{ij} S_{jj} w_j \bigr)/S_{ii}, \\
x_i & \leftarrow & w_i - \sum_{j > i} L_{ji} x_j, \\
y_{ij} & \leftarrow & S_{jj}^{-1} \delta_{ij} - \sum_{i \geq \ell > j} L_{\ell j} y_{i\ell} - \sum_{k > i} L_{k j} y_{k i}, \qquad \mbox{with $i  \geq j$ and $ij \in \mathcal{L}$}.
\eeqn
Using these local updates, we can make the following statement:

\begin{theorem}\label{thm:localconv}
The local update scheme of~\eqref{eq:iterupdate} converges to the unique solution of both the primary and secondary problems.  In the case of synchronous updates (all variables updated simultaneously), the scheme is guaranteed to converge in $2(|\mathcal{L}| +N)$ iterations where $N \leq |\mathcal{L}| \leq N(N+1)/2$.
\end{theorem}

\begin{proof}
The theorem follows directly from the fact that~\eqref{eq:iterupdate} is exactly solving~\eqref{eq:global} using forward and backward substitutions.  We have previously gone to great lengths to show that all the systems in~\eqref{eq:global} have either upper- or lower-triangular coefficient matrices.  While our iterative scheme is doing extra calculations, in the background it is also solving the global systems through forward and backward substitutions.  For example, in the case of solving for $\mbf{S}$ and $\mbf{L}$, after one update of $S_{11}$, it will be fixed and will no longer change.  After a subsequent update of $L_{21}$ it will be fixed and no longer change, then after a subsequent update of $S_{22}$ it will be fixed and will no longer change, and so on.  The $\mbf{S}$, $\mbf{L}$, and $\mbf{w}$ quantities all get incrementally locked in through a forward substitution, after which $\mbf{x}$ and $\mbf{y}$ get locked in through backward substitution.  As long as every variable gets updated infinitely often (even asynchronously in random order) we can guarantee convergence (at least probabilistically).  In the case of synchronous updates, the $\mbf{S}$ and $\mbf{L}$ quantities will be locked in after $|\mathcal{L}|$ iterations, the $\mbf{w}$ quantities after $N$ subsequent iterations (or sooner since happening in parallel with $\mbf{S}$ and $\mbf{L}$), the $\mbf{y}$ quantities after $|\mathcal{L}|$ more iterations, and the $\mbf{x}$ quantities after $N$ more iterations (or sooner since happening in parallel with $\mbf{y}$).  We can thus claim that the overall synchronous algorithm will converge in $2(|\mathcal{L}| +N)$ iterations; in practice the number will be lower as $\mbf{w}$ does not have to wait for $\mbf{S}$ and $\mbf{L}$ to converge and $\mbf{x}$ does not have to wait for $\mbf{y}$ to converge.
\end{proof}

Having a guarantee that our local updates in~\eqref{eq:iterupdate} converge to the global solution is quite advantageous in the case that parallel calculations can be performed.  The proof discusses how the global triangular systems are being solved `in the background' through forward/backward substitutions even if the local agents perform their updates in any order.   This comes at the cost of performing many `rough-draft' calculations that eventually get overwritten by the final answer as variables lock in their values.  In practice, these rough drafts may actually be quite close to the final answers and usable by consumers of the solutions.  As such, the next section details how to implement the updates as a local message-passing scheme.

\begin{figure}[t]
\includegraphics[width=0.95\textwidth]{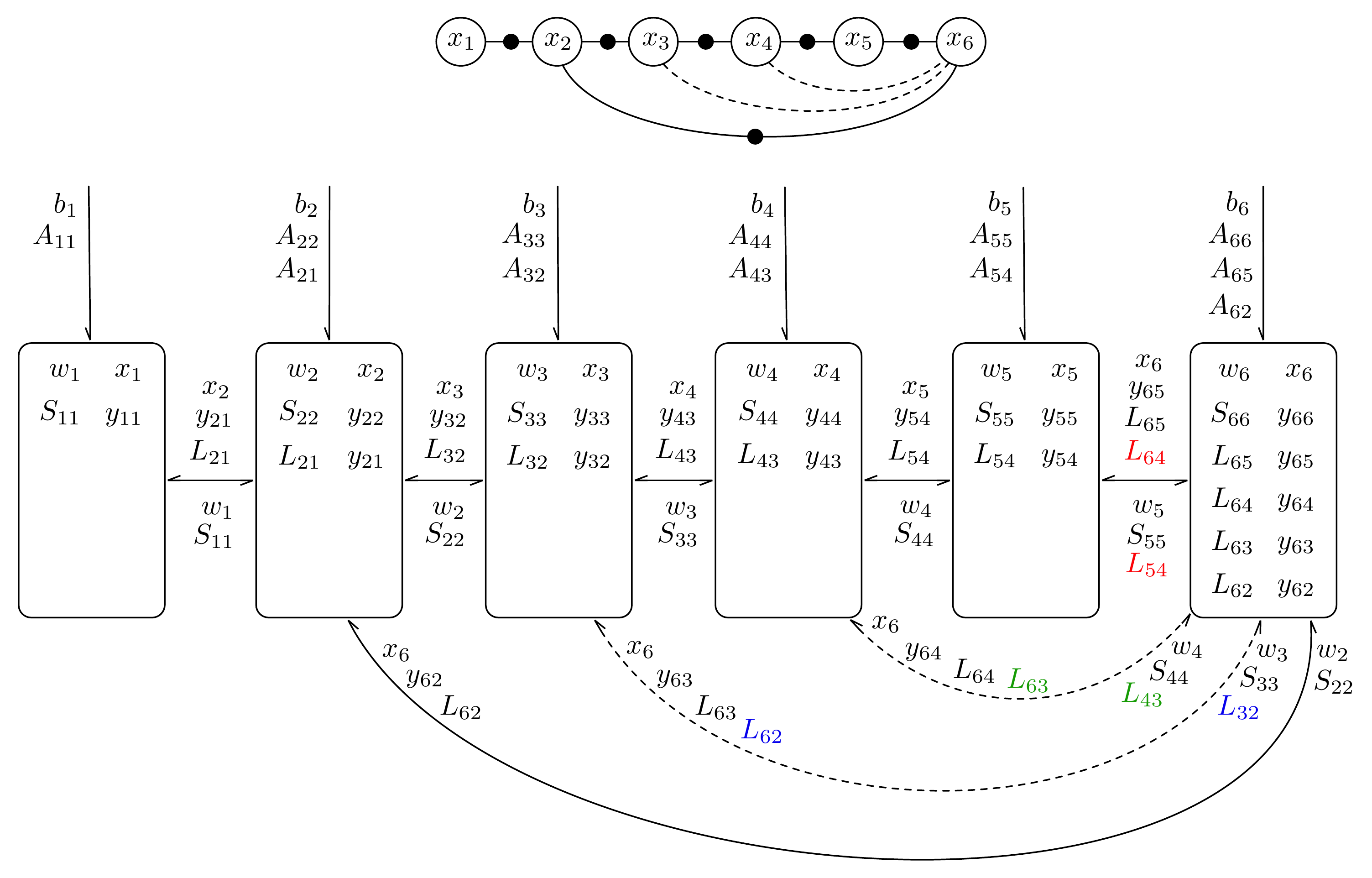}
\centering
\vspace*{-0.1in}
\caption{Example loopy factor graph showing $6$ variables and the local message-passing scheme from~\eqref{eq:iterupdate} that can be used to exactly solve both the primary and secondary problems.  Rounded rectangles are `agents' that look after the variables stored inside them.  Messages are exchanged along communication channels (solid channels derive from shared factors; dashed from extra channels based on the `four corners of a box' rule).  The messages in black are `standard' messages that follow a simple pattern based on whether the communication is to the left or right.  The additional red, green, and blue messages come in pairs; the pattern for their construction is if agents $i$ and $j$ are both in direct contact with agent $k$ to their lefts, these extra messages are exchanged.}
\label{fig:example2}
\end{figure}

\begin{figure}[t]
\includegraphics[width=\textwidth]{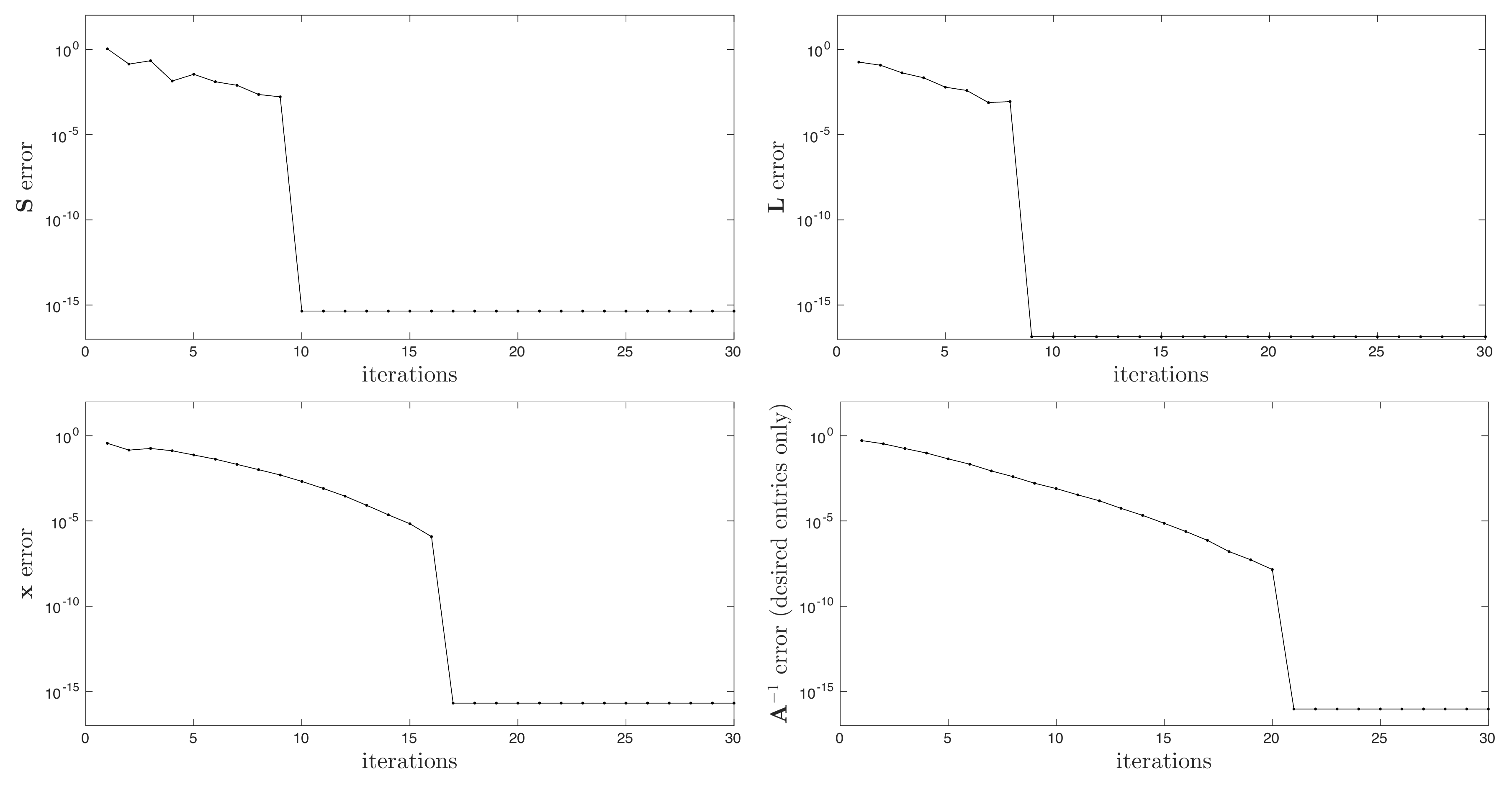}
\centering
\caption{Convergence of the norm of the errors of $\mbf{S}$, $\mbf{L}$, $\mbf{x}$, and $\mbf{A}^{-1}$ (desired entries only) for the example in 
Figure~\ref{fig:example2} using the local message passing scheme from~\eqref{eq:iterupdate}.  We see that all quantities converge to machine zero by about $20$ iterations. The sharp dropoffs in errors correspond to the forward or backward pass completing for that quantity.  However, even before this point the errors are reasonably small and could be usable by consumers of the solutions.}
\label{fig:example2_errors}
\end{figure}

\subsection{Parallel Calculation Using Message-Passing Agents}
 
 To understand how we can turn~\eqref{eq:iterupdate} into a local message-passing scheme, it is instructive to consider the toy example in Figure~\ref{fig:example2}.  In the rounded rectangles, we assemble $N$ groups of variables; we can think of these as individual agents that carry out their own computations in parallel while exchanging the messages shown along communication channels.  The example shows that the only communication channels needed are the ones between agents that are connected directly by a factor in the factor graph plus some extra channels that derive from the `four corners of a box' rule discussed earlier.  Appendix~\ref{app:updates} provides the detailed updates each agent must undertake.  Figure~\ref{fig:example2_errors} shows the convergence of the errors in $\mbf{S}$, $\mbf{L}$, $\mbf{x}$, and $\mbf{A}^{-1}$ (desired entries only) for this example.    Notably, by about $20$ iterations both the primary and secondary problem errors converge to machine zero.  However, even after a smaller number of iterations the errors are still quite small and hence the quantities could be used by consumers in their rough-draft state.

From the example, we can now try to generalize.  There are four issues that must be overcome to claim that we have a truly local scheme:
\begin{enumerate}
\item establishing agent order
\item communication graph discovery
\item calculation and storage responsibilities
\item message-exchange handshaking
\end{enumerate}
We will discuss a solution to each of these issues in turn.

\subsubsection{Establishing Agent Order}

We will see that all of these steps require establishing an {\em order} to the agents, which amounts to choosing the primary variable order in the global system.  We can think of this as a symmetry-breaking mechanism that allows the agents to decide how to handle loops in the factor graph.  There are a few options to do this:
\begin{enumerate}
\item we could use an oracle to make the selection
\item we could use an external signal such as a clock to order agents based on the time they join the collective
\item we could try to avoid all external influence and negotiate the order locally; for example, each agent could self-generate a large random integer (collisions could only be guaranteed probabilistically)
\end{enumerate}
In all cases, we assume the result of establishing the order is that all agents possess a unique integer, $i$.  The resulting order may not be the optimal one in terms of minimizing fill-in of the underlying $\mbf{L}$ matrix, but as shown in the example of Figure~\ref{fig:chain_fillin}, even a random order is highly unlikely to result in catastrophic fill-in.  Without loss of generality, we can assume the assigned integers are $1 \ldots N$ since we will only require agents to compare their integers with other agents to determine their actions.

\subsubsection{Communication Graph Discovery}

The communication graph (i.e., the local message-passing links between agents) can be constructed using two simple rules:
\begin{enumerate}
\item if agents $i$ and $j$ are both involved in the same factor in the factor graph, then a communication link between agents $i$ and $j$ is established
\item if agent $i$ is in communication with both agents $j > i$ and $k > i$, then agent $i$ tells agents $j$ and $k$ to establish a link 
\end{enumerate}
These rules must be applied recursively and result in a communication graph that corresponds to the fill-in pattern of $\mbf{L}$ since the second rule is the same as idea as the `four corners of a box' concept discussed in Lemma~\ref{lem:box}.

\subsubsection{Calculation and Storage Responsibilities}

Each agent $i$ needs to determine for which variables it has responsibility to store and calculate.  Each agent $i$ can do this locally using the following two rules:
\begin{enumerate}
\item agent $i$ takes responsible for $S_{ii}$, $w_i$, $x_i$, and $y_{ii}$
\item if agent $i$ is in communication with agent $j < i$, then agent $i$ takes responsibility for $L_{ij}$, $y_{ij}$
\end{enumerate}
Once the communication graph is established, this scheme fully assigns responsibility for every variable in $\mbf{S}$, $\mbf{L}$, $\mbf{w}$, $\mbf{x}$, and $\mbf{y}$ to only one agent.

\subsubsection{Message-Exchange Handshaking}

The goal of message-exchange handshaking is to establish which messages are to be exchanged along each communication link.  These can also be determined locally by agent $i$ using the following rules:
\begin{enumerate}
\item if agent $i$ is in communication with agent $j > i$, then agent $i$ pushes $S_{ii}$ and $w_i$ to agent $j$
\item if agent $i$ is in communication with agent $j < i$, then agent $i$ pushes $L_{ij}$, $y_{ij}$, and $x_i$ to agent $j$
\item if agent $i$ is in communication with agent $j < i$ and $k < j$, then agent $i$ pushes $L_{ik}$ to agent $j$ and pulls $L_{jk}$ from agent $j$
\end{enumerate}
These rules ensure each agent has access to all the required quantities to perform the updates of the variables for which it is responsible.

\medskip
Given these simple rules, a communication graph can be established, responsibilities for storage and calculation assigned, and message protocols established in a completely local manner while also ensuring both the primary and secondary problems are solved exactly.

\section{Related Work}
\label{sec:relatedwork}

Our analysis of the result of \citet{takahashi73} is new but the end result is mostly the same.  The use of elimination and duplication matrices was motivated by the work of \citet{magnus19}, who have been exploring their uses since the late 1970s \citep{magnus79,magnus80}; they have also used these tools to represent the upper- or lower-half of a covariance or inverse covariance matrix.  To the best of our knowledge, the result of Takahashi et al. has not be presented using Kronecker algebra and these tools, although the fit seems quite natural.  One key aspect our presentation affords is that once the linear system in~\eqref{eq:utvoss} is built, the upper-triangular nature of the $\mbf{C}_{22}$ matrix provides a clear order for the calculation of the desired entries of $\mbf{A}^{-1}$ through backward substitution.  Casting the triangular factorization into the same Kronecker form is also quite useful as it cleanly leads to the iterative scheme that enables the local message-passing solution.

Within robotics and computer vision, \citet[App. B.4]{triggs00} and \citet{kaess09} discuss methods to calculate specific blocks of the covariance matrix efficiently from the inverse covariance for computer vision and robotics applications, but do not discuss doing so for the complete set of covariance blocks corresponding to the non-zero blocks of the inverse covariance matrix.  \citet{mahon08} appear to be the first to use the Takahashi et al. approach to extract covariance information in robotics for a visual \ac{SLAM} problem.  \citet{barfoot_ijrr20} later show how to use the Takaashi et al. result within an iterative variational inference framework, as mentioned in the introduction of the current paper.  None of these works looks at how a local message-passing scheme could be used to solve both the primary and secondary problems.

{\em Belief propagation} (also known as sum-product message passing) was introduced by \citet{pearle88} to perform inference on graphical models;  \citet{koller09} provides a more recent exposition.  \ac{GaBP} restricts the beliefs involved to Gaussians and is the common alternative to the message-passing scheme we discuss in this paper.  \citet{davison19} recently proposed the use of \ac{GaBP} to solve large-scale localization and mapping problems in a decentralized manner, renewing interest in the method.  Belief propagation is known to converge to the correct global solution to the inference problem when the graphical model is a tree; however, {\em loopy belief propagation} may or may not converge and various sufficient conditions have been established for convergence.  \citet{weiss00} provide the starting point by showing that \ac{GaBP} will converge (mean only not covariance) under the condition that $\mbf{A}$ is diagonally dominant.  To our knowledge, no condition has been derived for \ac{GaBP} to ensure the covariance quantities will converge to the global solution in the case of a loopy graph.  Some of the analyses of \ac{GaBP} look to view the method as an iterative solver for a linear system of equations (e.g., Jacobi's algorithm), a connection first pointed out by \citet{weiss00}. However,  as there have been approximations made, \ac{GaBP} cannot be shown in general to be solving the primary (and secondary) linear systems discussed in this paper.  Our paper in a sense follows this connection in the opposite direction by starting with the known result of \citet{takahashi73} and then showing what it would take to implement this as a local message-passing scheme.  While we can guarantee convergence in both the mean and covariance quantities, this comes at the cost of additional communication links, memory, and computation compared to \ac{GaBP}.  There have been other attempts to augment \ac{GaBP} to ensure convergence \citep{plarre00,du17,marelli20} as well as discussion of general convergent message-passing algorithms \citep{hazan12,meltzer12,thomas19}; however, the connection is not made to Takahashi et al.'s classic result for the Gaussian case.  \citet{bickson09} provides a good overview of \ac{GaBP} and discusses how it can be used as a standard linear algebra solver \citep{shental08}.

More generally, {\em selected inversion} considers solving for specific entries of $\mbf{A}^{-1}$ even when it is not a covariance matrix, although then some of the structure of the problem no longer applies.  \citet{rouet09} and more recently \citet{verbosio19} provide nice overviews of different methods used to solve this problem.  \citet{jacquelin18} provides a method that parallelizes the left-looking selected inversion algorithm on a shared memory system.  \citet{campbell95} combine the Takahashi et al. result with multifrontal methods to solve the selected inversion problem in a serial manner very efficiently.  \citet{lin11} also seek to exploit Takahashi et al. to rapidly invert sparse matrices using a relative index array in an algorithm called {\em SelInv}.  \citet{siden18} discuss a fast Rao-Blackwellized Monte Carlo sampling-based method to find approximations to the covariance matrix when its inverse is sparse.  \citet{betancourt86} extend the classic Takahashi et al. result by showing how to the compute the full inverse in the same amount of time as the sparse inverse through the use of parallel operations.  

\section{Conclusion and Future Work}
\label{sec:conclusion}

We have provided motivation for and a clear statement of the \ac{FLAPOGI} problem, which is of course not new.  This problem is important in Gaussian estimation even when the full Bayesian posterior is not Gaussian (e.g., in the \ac{ESGVI} framework).  Our novel contributions are twofold.  First, we provide a novel presentation of the \citet{takahashi73} solution to the problem.  Second, we work out the details necessary to implement the method using a purely local message-passing scheme between a collection of agents computing in parallel.  The result is that we can guarantee convergence on both the primary (mean) and secondary (covariance) problems, even in the presence of loops in the associated factor graph, in a finite number of iterations for synchronous updates.

There are several avenues that could be explored in future work.  Foremost, we would like to implement the proposed local message-passing scheme on a large real-world inference problem to evaluate its performance; here our aim was merely to lay the foundation.  Also, we avoided discussion of how to best select the primary variable order, which influences the amount of fill-in in the global matrices and the local message-passing communication graph; we believe it may be possible to combine the result of \citet{takahashi73} with the modern solvers used for inference in robotics and computer vision \citep{kaess08,kaess11} to this aim.  It would also be interesting to look more deeply at the connection to Gaussian belief propagation; the methods perform the same on a tree graph, but what approximations are needed in the loopy case for our message-passing scheme to be similar/equivalent to \ac{GaBP} (e.g., perhaps we simply need to delete the extra communication channels our method requires)?  We hope the ideas presented herein provide a starting point for further investigation.

\bibliographystyle{asrl}
\bibliography{refs,refs2,book,pubs}

\appendix

\section{Kronecker Product and Vectorization}
\label{app:kronvec}

There are several identities of which we make use involving the {\em Kronecker product} $\otimes$ and the {\em vectorization operator} $\vec(\cdot)$ that stacks the columns of a matrix:
\begin{equation}\label{kron:1}
\begin{aligned}
\vec(\mbf{a}) & \equiv  \mbf{a} \\
\vec(\mbf{a}\mbf{b}^T) & \equiv  \mbf{b} \otimes \mbf{a} \\
\vec(\mbf{A}\mbf{B}\mbf{C}) & \equiv  (\mbf{C}^T \otimes \mbf{A} )\, \vec(\mbf{B}) \\
\vec(\mbf{A})^T \vec(\mbf{B}) & \equiv  \tr(\mbf{A}^T\mbf{B}) \\
(\mbf{A} \otimes \mbf{B})(\mbf{C} \otimes \mbf{D}) & \equiv  (\mbf{A}\mbf{C}) \otimes (\mbf{B}\mbf{D}) \\
(\mbf{A} \otimes \mbf{B})^{-1} & \equiv  \mbf{A}^{-1} \otimes \mbf{B}^{-1} \\
 (\mbf{A} \otimes \mbf{B})^{T} & \equiv  \mbf{A}^{T} \otimes \mbf{B}^{T} \\
 \left| \mbf{A} \otimes \mbf{B} \right| & \equiv \left| \mbf{A} \right|^M \left| \mbf{B} \right|^N \qquad (\mbox{$\mbf{A}$ is $N\times N$, $\mbf{B}$ is $M \times M$}) \\
\mbox{tr} \left( \mbf{A} \otimes \mbf{B} \right) & \equiv \mbox{tr} (\mbf{A}) \, \mbox{tr} (\mbf{B})
\end{aligned}
\end{equation}
It is worth noting that $\otimes$ and $\vec(\cdot)$ are linear operators.  

 \section{Elimination and Duplication Matrices}
 \label{app:elimdup}

To handle symmetric and triangular matrices, we follow \citet{magnus19} and define {\em elimination} and {\em duplication} matrices as
\begin{equation}
\mbf{E} = \bbm \vdots \\ \vec(\mbf{1}_{ji})^T \\ \vdots \ebm_{i \geq j}, \quad \mbf{D} = \bbm & \cdots & \vec( (1 - \delta_{ij}) \mbf{1}_{ij} + \mbf{1}_{ji} ) & \cdots & \ebm_{i \geq j},
\end{equation}
where $\delta_{ij}$ is the Kronecker delta and $\mbf{1}_{ij}$ is an $N \times N$ matrix of zeros except for the $ij$ entry, which is one.  The elimination matrix selects only the upper-half entries of a vectorized matrix; it can be used to remove duplicate entries in a vectorized symmetric matrix or to remove the zeros of a vectorized upper-triangular matrix.  The duplication matrix can be used to reconstitute a symmetric matrix from its eliminated version.  Note, we must pick the same variable order, $ij$, for $\mbf{E}$ and $\mbf{D}$.  Other definitions of elimination matrices are possible but this is the one that is useful in our situation.

For example with $N=3$ we can set
\begin{equation}
\mbf{E} = \bbm 1 & 0 & 0 & 0 & 0 & 0 & 0 & 0 & 0 \\
		        0 & 0 & 0 & 1 & 0 & 0 & 0 & 0 & 0 \\
		        0 & 0 & 0 & 0 & 1 & 0 & 0 & 0 & 0 \\
		        0 & 0 & 0 & 0 & 0 & 0 & 1 & 0 & 0 \\
		        0 & 0 & 0 & 0 & 0 & 0 & 0 & 1 & 0 \\
		        0 & 0 & 0 & 0 & 0 & 0 & 0 & 0 & 1 \ebm, \quad
\mbf{D} = \bbm 1 & 0 & 0 & 0 & 0 & 0 \\
			0 & 1 & 0 & 0 & 0 & 0 \\
			0 & 0 & 0 & 1 & 0 & 0 \\
			0 & 1 & 0 & 0 & 0 & 0 \\
			0 & 0 & 1 & 0 & 0 & 0 \\
			0 & 0 & 0 & 0 & 1 & 0 \\
			0 & 0 & 0 & 1 & 0 & 0 \\
			0 & 0 & 0 & 0 & 1 & 0 \\
			0 & 0 & 0 & 0 & 0 & 1 \ebm,
\end{equation}
for a variable order that we refer to as basic triangular.

As mentioned above, the duplication matrix reconstitutes a vectorized symmetric matrix from only its upper-half entries:
\begin{equation}
\mbf{D} \mbf{E} \vec(\mbf{S}) = \vec(\mbf{S}),
\end{equation}
for a symmetric matrix, $\mbf{S}$.    It is also always true that
\begin{equation}
\mbf{E} \mbf{D} = \mbf{1},
\end{equation}
with $\mbf{1}$ the identity matrix.  Premultiplying a vectorized lower-triangular matrix, $\mbf{L}$ by the elimination matrix retains only the diagonal entries:
\begin{equation}
\mbf{E} \vec(\mbf{L}) = \mbf{E} \vec(\mbox{diag}(\mbf{L})).
\end{equation}
For an upper-triangular matrix, $\mbf{U}$, we have the helpful expression,
\begin{equation}
\mbf{E}^T \mbf{E} \vec(\mbf{U}) = \vec(\mbf{U}),
\end{equation}
of which we will make use at times.  We also have that
\begin{equation}
\mbf{E} \left( \mbf{1} \otimes \mbf{L} \right) \mbf{D} \mbf{E} = \mbf{E} \left( \mbf{1} \otimes \mbf{L} \right),
\end{equation}
for a lower-triangular matrix, $\mbf{L}$.  The following inverse expression also holds:
\begin{equation}
\left(\mbf{E} \left( \mbf{1} \otimes \mbf{L} \right) \mbf{D}\right)^{-1}  = \mbf{E} \left( \mbf{1} \otimes \mbf{L}^{-1} \right) \mbf{D},
\end{equation}
where again $\mbf{L}$ is lower triangular.

\section{Detailed Updates for Example in Figure~\ref{fig:example2}}
\label{app:updates}

Below are the detailed local calculations for the example in Figure~\ref{fig:example2}, broken down by agent:

\noindent
{\em Agent 1:}
\beqn{}
S_{11} & \leftarrow & A_{11} \\
w_1 & \leftarrow & b_1 / S_{11}   \\
x_1 & \leftarrow &   w_1 - L_{21} x_2   \\
y_{11} & \leftarrow &  1/S_{11} - L_{21} y_{21}  \hspace*{4.2in}
\eeqn

\medskip
\noindent
{\em Agent 2:}
\beqn{}
S_{22} & \leftarrow & A_{22} - L_{21}^2 S_{11} \\
L_{21} & \leftarrow &  A_{21} / S_{11} \\
w_2 & \leftarrow &  (b_2 - L_{21} S_{11} w_1 ) / S_{22}  \\
x_2 & \leftarrow &   w_2 - L_{32} x_3 - L_{62} x_6  \\
y_{22} & \leftarrow &  1/S_{22} - L_{32} y_{32} - L_{62} y_{62}     \\
y_{21} & \leftarrow &  -L_{21} y_{22} \hspace*{4.6in}
\eeqn

\medskip
\noindent
{\em Agent 3:}
\beqn{}
S_{33} & \leftarrow & A_{33} - L_{32}^2 S_{22} \\
L_{32} & \leftarrow & A_{32} / S_{22} \\
w_3 & \leftarrow &   (b_3 - L_{32} S_{22} w_2 ) / S_{33}  \\
x_3 & \leftarrow &  w_3 - L_{43} x_4 - L_{63} x_6   \\
y_{33} & \leftarrow &  1/S_{33} - L_{43} y_{43} - L_{63} y_{63}     \\
y_{32} & \leftarrow &   -L_{32} y_{33} - L_{62} y_{63} \hspace*{4in}
\eeqn

\medskip
\noindent
{\em Agent 4:}
\beqn{}
S_{44} & \leftarrow & A_{44} - L_{43}^2 S_{33} \\
L_{43} & \leftarrow & A_{43} / S_{33} \\
w_4 & \leftarrow &  (b_4 - L_{43} S_{33} w_3 ) / S_{44}  \\
x_4 & \leftarrow &  w_4 - L_{54} x_5 - L_{64} x_6   \\
y_{44} & \leftarrow &  1/S_{44} - L_{54} y_{54} - L_{64} y_{64}     \\
y_{43} & \leftarrow &   -L_{43} y_{44} - L_{63} y_{64} \hspace*{4in}
\eeqn

\medskip
\noindent
{\em Agent 5:}
\beqn{}
S_{55} & \leftarrow & A_{55} - L_{54}^2 S_{44} \\
L_{54} & \leftarrow & A_{54} / S_{44} \\
w_5 & \leftarrow &  (b_5 - L_{54} S_{44} w_4 ) / S_{55}  \\
x_5 & \leftarrow &  w_5 - L_{65} x_6   \\
y_{55} & \leftarrow &   1/S_{55} - L_{65} y_{65}    \\
y_{54} & \leftarrow &    -L_{54} y_{55} - L_{64} y_{65} \hspace*{4in}
\eeqn

\medskip
\noindent
{\em Agent 6:}
\beqn{}
S_{66} & \leftarrow & A_{66} - L_{65}^2 S_{55} - L_{64}^2 S_{44} - L_{63}^2 S_{33} - L_{62}^2 S_{22} \\
L_{65} & \leftarrow & (A_{65} - L_{64} L_{54} S_{44} ) / S_{55} \\
L_{64} & \leftarrow &  -L_{63} L_{43} S_{33} / S_{44} \\
L_{63} & \leftarrow & -L_{62} L_{32} S_{22} / S_{33} \\
L_{62} & \leftarrow &  A_{62} / S_{22} \\
w_6 & \leftarrow & (b_6 - L_{65} S_{55} w_5 - L_{64} S_{44} w_4 - L_{63} S_{33} w_3 - L_{62} S_{22} w_2 ) / S_{66}  \\
x_6 & \leftarrow &  w_6   \\
y_{66} & \leftarrow &   1/S_{66}    \\
y_{65} & \leftarrow &  -L_{65} y_{66}  \\
y_{64} & \leftarrow &  -L_{64} y_{66} - L_{54} y_{65}  \\
y_{63} & \leftarrow &  -L_{63} y_{66} - L_{43} y_{64}  \\
y_{62} & \leftarrow &  -L_{62} y_{66} - L_{32} y_{63}  \hspace*{4in}
\eeqn

\end{document}

%% file: notation_header.tex
\newcommand{\bbm}{\begin{bmatrix}}
\newcommand{\ebm}{\end{bmatrix}}
\newcommand{\mbf}{\mathbf}
\newcommand{\mbs}[1]{{\boldsymbol{#1}}}

\newcommand{\beq}{\begin{equation}}
\newcommand{\eeq}{\end{equation}}
\newcommand{\bdis}{\begin{displaymath}}
\newcommand{\edis}{\end{displaymath}}
\newcommand{\beqn}[1]{\begin{subequations}\label{eq:#1}\begin{eqnarray}}
\newcommand{\eeqn}{\end{eqnarray}\end{subequations}}
\renewcommand{\vec}{\mbox{vec}}

\newcommand{\tr}{\mbox{tr}}